\newif\ifJOURNAL
\JOURNALfalse
\newif\ifCONF
\CONFfalse
\newif\ifarXiv
\arXivfalse
\newif\ifWP
\WPfalse
\newif\ifFULL
\FULLfalse

\arXivtrue

\newif\ifnotCONF	
\notCONFtrue
\ifCONF\notCONFfalse\fi

\newif\ifTR		
\TRfalse
\ifarXiv\TRtrue\fi
\ifWP\TRtrue\fi
\newif\ifnotTR
\notTRtrue
\ifarXiv\notTRfalse\fi
\ifWP\notTRfalse\fi

\ifCONF
  
\fi
\ifarXiv
  
\fi
\ifWP
  
\fi

\ifCONF
  \documentclass[anon]{colt2014} 

  \title[Efficiency of conformalized ridge regression]{Efficiency of conformalized ridge regression}
  \usepackage{times}

  \coltauthor{\Name{Evgeny Burnaev} \Email{evgeny.burnaev@datadvance.net}\\
    \addr Datadvance, Institute of Problems of Information Transmission,
      and Moscow Institute of Physics and Technology, Russia
    \AND
    \Name{Vladimir Vovk} \Email{v.vovk@rhul.ac.uk}\\
    \addr Department of Computer Science,
      Royal Holloway, University of London, UK
  }
\fi

\ifarXiv
  \documentclass[11pt]{article}
  \usepackage{amsmath,amsthm,amsfonts,amssymb,latexsym,graphicx}
  \newcommand{\Extra}[1]{}
\fi

\ifWP
  \documentclass[10pt]{article}
  \usepackage{amsmath,amsthm,amsfonts,amssymb,latexsym,graphicx,url}
  \input{OT2enc.def}

  \input{/Doc/Computing/Latex/kp.txt}
  \newcommand{\Extra}[1]{}
\fi

\ifFULL
  \usepackage{color}
  \renewcommand{\Extra}[1]{\blue{#1}}
  \newcommand{\blue}[1]{\textcolor{blue}{#1}}
  \newcommand{\bluebegin}{\begingroup\color{blue}}
  \newcommand{\blueend}{\endgroup}

\fi

\ifnotCONF
  \newcommand{\citep}[1]{\cite{#1}}
  \newcommand{\citet}[1]{\cite{#1}}
  
  \newcommand{\citealt}[1]{\cite{#1}}
  \newcommand{\citeyear}[1]{\cite{#1}}
\fi

\allowdisplaybreaks[4]

\ifnotCONF
  \newtheorem{theorem}{Theorem}
  \newtheorem{lemma}[theorem]{Lemma}
  \newtheorem{corollary}[theorem]{Corollary}
  \newtheorem{proposition}[theorem]{Proposition}
  \theoremstyle{definition}
  \newtheorem*{remark}{Remark}
\fi

\DeclareMathOperator{\Prob}{\mathbb{P}}
\DeclareMathOperator{\Expect}{\mathbb{E}}
\DeclareMathOperator{\var}{\textrm{var}}
\DeclareMathOperator{\cov}{\textrm{cov}}
\newcommand{\law}{\stackrel{{\rm law}}{\longrightarrow}}

\ifarXiv
  \title{Efficiency of conformalized ridge regression}
  \author{Evgeny Burnaev\\
    \texttt{evgeny.burnaev@datadvance.net}\\
    \texttt{burnaevevgeny@gmail.com}\\[3mm]
  Vladimir Vovk\\
    \texttt{v.vovk@rhul.ac.uk}}
\fi

\ifWP
  \title{Efficiency of conformalized ridge regression}
  \author{Evgeny Burnaev\\
    \texttt{evgeny.burnaev@datadvance.net}\\
    \texttt{burnaevevgeny@gmail.com}\\[3mm]
  Vladimir Vovk\\
    \texttt{v.vovk@rhul.ac.uk}}
  
  \twodatestrue
  
\fi

\begin{document}
\maketitle
\begin{abstract}
  Conformal prediction is a method of producing prediction sets
  that can be applied on top of a wide range of prediction algorithms.
  The method has a guaranteed coverage probability under the standard IID assumption
  regardless of whether the assumptions (often considerably more restrictive)
  of the underlying algorithm are satisfied.
  However, for the method to be really useful it is desirable
  that in the case where the assumptions of the underlying algorithm
  are satisfied,
  the conformal predictor loses little in efficiency
  as compared with the underlying algorithm
  (whereas being a conformal predictor, it has the stronger guarantee of validity).
  In this paper we explore the degree to which this additional requirement of efficiency
  is satisfied in the case of Bayesian ridge regression;
  we find that asymptotically conformal prediction sets differ little
  from ridge regression prediction intervals
  when the standard Bayesian assumptions are satisfied.
\end{abstract}

\ifCONF
\fi

\section{Introduction}

This paper discusses theoretical properties of the procedure
described in the abstract as applied to Bayesian ridge regression in the primal form.
The procedure itself has been discussed earlier in the Bayesian context under the names
of frequentizing (\citealt{wasserman:2011}, Section~3)
and de-Bayesing (\citealt{vovk/etal:2005book}, p.~101);
in this paper, however, we prefer the name ``conformalizing''.
The procedure has also been studied empirically
(see, e.g., \citealt{vovk/etal:2005book}, Figures~10.1--10.5,
and \citealt{wasserman:2011}, Figure~1,
corrected in \citealt{vovk:2013vapnik}, Figure~11.1).
To our knowledge,
this paper is the first to explore the procedure theoretically.

The purpose of conformalizing is to make prediction algorithms,
first of all Bayesian algorithms, valid under the assumption
that the observations are generated independently from the same probability measure;
we will refer to this assumption as the IID assumption.
This is obviously a desirable step provided that we do not lose much
if the assumptions of the original algorithm happen to be satisfied.
The situation here resembles that in nonparametric hypothesis testing
(see, e.g., \citealt{randles/etal:2004}),
where nonparametric analogues of some classical parametric tests
relying on Gaussian assumptions
turned out to be surprisingly efficient
even when the Gaussian assumptions are satisfied.

We start the main part of the paper from Section~\ref{sec:BRR},
in which we define the ridge regression procedure
and the corresponding prediction intervals
in a Bayesian setting involving strong Gaussian assumptions.
It contains standard material and so no proofs.
The following section, Section~\ref{sec:CRR},
applies the conformalizing procedure to ridge regression
in a way that facilitates theoretical analysis
in the following sections;
the resulting ``conformalized ridge regression''
is similar to but somewhat different from the algorithm
called ``ridge regression confidence machine''
in \citet{vovk/etal:2005book}.

Section~\ref{sec:result} contains our main result.
It shows that asymptotically we lose little when we conformalize
ridge regression and the Gaussian assumptions are satisfied;
namely, conformalizing changes the prediction interval by $O(n^{-1/2})$ with high probability,
where $n$ is the number of observations.
Our main result gives precise asymptotic distributions
for the differences between the left and right end-points
of the prediction intervals output by the Bayesian and conformal predictors.
These are theoretical counterparts of the preliminary empirical results
obtained in \citet{vovk/etal:2005book}
(Figures~10.1--10.5 and Section~8.5, pp.~205--207) and \citet{vovk/etal:2009AOS}.
We then discuss and interpret our main result
using the notions of efficiency and conditional validity
(introduced in the previous two sections).
Section~\ref{sec:details} gives a more explicit description of conformalized ridge regression,
and in Section~\ref{sec:proof} we prove the main result.

Other recent theoretical work about efficiency and conditional validity
of conformal predictors includes Lei and Wasserman's
\ifCONF(\citeyear{lei/wasserman:2013})\fi
\ifnotCONF\cite{lei/wasserman:2013}\fi.
Whereas our predictor is obtained by conformalizing ridge regression,
Lei and Wasserman's conformal predictor is specially crafted
to achieve asymptotic efficiency and conditional validity.
It is intuitively clear that whereas our algorithm
is likely to produce reasonable results in practice
(in situations where ridge regression produces reasonable results),
Lei and Wasserman's algorithm is primarily of theoretical interest.
A significant advantage of their algorithm, however,
is that it is guaranteed to be asymptotically efficient
and conditionally valid under their regularity assumptions,
whereas our algorithm is guaranteed to be asymptotically efficient
and conditionally valid only under the Gaussian assumptions.

\section{Bayesian ridge regression}
\label{sec:BRR}

Much of the notation introduced in this section will be used throughout the paper.
We are given a training sequence $(x_1,y_1),\ldots,(x_{n-1},y_{n-1})$
and a test object $x_{n}$,
and our goal is to predict its label $y_{n}$.
Each \emph{observation} $(x_i,y_i)$, $i=1,\ldots,n$
consists of an \emph{object} $x_i\in\mathbb{R}^p$ and a \emph{label} $y_i\in\mathbb{R}$.
We are interested in the case where the number $n-1$ of training observations is large,
whereas the number $p$ of attributes is fixed.
Our setting is probabilistic;
in particular,
the observations are generated by a probability measure.

In this section we do not assume anything about the distribution of the objects $x_1,\ldots,x_n$,
but given the objects, the labels $y_1,\ldots,y_n$ are generated by the rule
\begin{equation}\label{eq:model}
  y_i = w \cdot x_i + \xi_i,
\end{equation}
where $w$ is a random vector distributed as $N(0,(\sigma^2/a)I)$
(the Gaussian distribution being parameterized by its mean and covariance matrix,
and $I:=I_p$ being the unit $p\times p$ matrix),
each $\xi_i$ is distributed as $N(0,\sigma^2)$,
the random elements $w,\xi_1,\ldots,\xi_n$ are independent (given the objects),
and $\sigma$ and $a$ are given positive numbers.

The conditional distribution for the label $y_{n}$ of the test object $x_{n}$
given the training sequence and $x_n$ is
\begin{equation*} 
  N
  \left(
    \hat y_{n},
    (1+g_n) \sigma^2
  \right),
\end{equation*}
where
\begin{align}
  \hat y_n
  &:=
  x'_{n}(X'X+aI)^{-1}X'Y,
  \label{eq:RR}\\
  g_n
  &:=
  x_{n}' (X'X+aI)^{-1} x_{n},
  \label{eq:g-n}
\end{align}
$X=X_{n-1}$ is the design matrix for the training sequence
(the $(n-1)\times p$ matrix whose $i$th row is $x'_i$, $i=1,\ldots,n-1$),
and $Y=Y_{n-1}$ is the vector $(y_1,\ldots,y_{n-1})'$ of the training labels;
see, e.g., \citet{vovk/etal:2005book}, (10.24).
Therefore, the Bayesian prediction interval is
\begin{equation}\label{eq:Bayes}
  (B_*,B^*)
  :=
  \Bigl(
    \hat y_{n}
    -
    \sqrt{1+g_{n}}
    \sigma
    z_{\epsilon/2},
    \hat y_{n}
    +
    \sqrt{1+g_{n}}
    \sigma
    z_{\epsilon/2}
  \Bigr),
\end{equation}
where $\epsilon$ is the significance level
(the permitted probability of error,
so that $1-\epsilon$ is the required coverage probability)
and $z_{\epsilon/2}$ is the $(1-\epsilon/2)$-quantile
of the standard normal distribution $N(0,1)$.

The prediction interval (\ref{eq:Bayes}) enjoys several desiderata:
it is unconditionally valid,
in the sense that its error probability is equal to the given significance level $\epsilon$;
it is also valid conditionally on the training sequence and the test object $x_n$;
finally, this prediction interval is the shortest possible conditionally valid interval.
We will refer to the class of algorithms producing prediction intervals (\ref{eq:Bayes})
(and depending on the parameters $\sigma$ and $a$)
as \emph{Bayesian ridge regression} (BRR).

\section{Conformalized ridge regression}
\label{sec:CRR}

Conformalized ridge regression (CRR) is a special case of conformal predictors;
the latter are defined in, e.g., \citet{vovk/etal:2005book}, Chapter~2,
but we will reproduce the definition in our current context.
  First we define the \emph{CRR conformity measure} $A$ as the function
  that maps any finite sequence $(x_1,y_1),\ldots,(x_n,y_n)$ of observations
  of any length $n$
  to the sequence $(\alpha_1,\ldots,\alpha_n)$
  of the following \emph{conformity scores} $\alpha_i$:
  for each $i=1,\ldots,n$,
  \begin{equation*} 
    \alpha_i
    :=
    \left|
      \left\{
        j=1,\ldots,n \mid r_j\ge r_i
      \right\}
    \right|
    \wedge
    \left|
      \left\{
        j=1,\ldots,n \mid r_j\le r_i
      \right\}
    \right|,
  \end{equation*}
  where $(r_1,\ldots,r_n)'$ is the vector of ridge regression residuals
  $r_i:=y_i-\hat y_i$,
  $$
    \hat y_i:=x_i'(X'_nX_n+aI)^{-1}X'_nY_n
  $$
  (cf.\ (\ref{eq:RR})),
  $X_n$ is the overall design matrix
  (the $n\times p$ matrix whose $i$th row is $x'_i$, $i=1,\ldots,n$),
  and $Y_n$ is the overall vector of labels
  (the vector in $\mathbb{R}^n$ whose $i$th element is $y_i$, $i=1,\ldots,n$).

\begin{remark}
  \ifCONF\begingroup\rm\fi
  We interpret $\alpha_i$ as the degree to which the element $(x_i,y_i)$
  conforms to the full sequence $(x_1,y_1),\ldots,(x_n,y_n)$.
  Intuitively, $(x_i,y_i)$ conforms to the sequence
  if its ridge regression residual
  is neither among the largest nor among the smallest.
  Instead of the simple residuals $r_i$ we could have used deleted or studentized residuals
  (see, e.g., \citealt{vovk/etal:2005book}, pp.~34--35),
  but we choose the simplest definition, which makes calculations feasible.
  Another possibility is to use $-\left|r_i\right|$ as conformity scores;
  this choice leads to what was called ``ridge regression confidence machines''
  in \citet{vovk/etal:2005book}, Chapter~2,
  but its analysis is less feasible.
  \ifCONF\endgroup\fi
\end{remark}

Given a significance level $\epsilon\in(0,1)$,
a training sequence
\ifCONF$(x_1,y_1),\ldots,(x_{n-1},y_{n-1})$, \fi
\ifnotCONF$$(x_1,y_1),\ldots,(x_{n-1},y_{n-1}),$$\fi%
and a test object $x_n$,
\emph{conformalized ridge regression} outputs the prediction set
\begin{equation}\label{eq:prediction-set}
  \Gamma
  :=
  \left\{
    y \mid p^y>\epsilon
  \right\},
\end{equation}
where the \emph{p-values} $p^y$ are defined by
$$
  p^y
  :=
  \frac
  {\left|\left\{i=1,\ldots,n\mid\alpha^y_{i}\le\alpha^y_{n}\right\}\right|}
  {n}
$$
and the conformity scores $\alpha^y_i$ are defined by
\begin{equation}\label{eq:alpha-1}
  (\alpha^y_1,\ldots,\alpha^y_n)
  :=
  A
  \bigl(
    (x_1,y_1),\ldots,(x_{n-1},y_{n-1}),(x_n,y)
  \bigr).
\end{equation}
Define the \emph{prediction interval} output by CRR
as the closure of the convex hull of the prediction set $\Gamma$;
we will use the notation $C_*$ and $C^*$
for the left and right end-points of this interval, respectively.
(Later we will introduce assumptions that will guarantee that $\Gamma$ itself is an interval
from some $n$ on.)
As discussed later in Section~\ref{sec:details},
CRR is computationally efficient:
e.g., its computation time is $O(n\ln n)$ in the on-line mode.

CRR relies on different assumptions about the data
as compared with BRR.
Instead of the Gaussian model~(\ref{eq:model}),
where $\xi_i\sim N(0,\sigma^2)$ and $w\sim N(0,(\sigma^2/a)I)$,
it uses the assumption that is standard in machine learning:
we consider observations $(x_1,y_1),\ldots,(x_n,y_n)$ that are IID
(independent and identically distributed).

\begin{proposition}[\citealt{vovk/etal:2005book}, Proposition~2.3]\label{prop:validity}
  If $(x_1,y_1),\ldots,(x_n,y_n)$ are IID observations,
  the coverage probability of CRR 
  (i.e., the probability of $y_n\in\Gamma$,
  where $\Gamma$ is defined by~(\ref{eq:prediction-set}))
  is at least $1-\epsilon$.
\end{proposition}

Proposition~\ref{prop:validity} asserts the unconditional validity of CRR.
Its validity conditional on the training sequence and the test object is not, however, guaranteed
(and it is intuitively clear that ensuring validity conditional on the test object
prevents us from relying on the IID assumption about the objects).
For a discussion of conditional validity in the context of conformal prediction,
see \citet{lei/wasserman:2013}, Section~2, and, more generally, \citet{vovk:2013ML}.
Efficiency (narrowness of the prediction intervals) is not guaranteed either.

The kind of validity asserted in Proposition~\ref{prop:validity}
is sometimes called ``conservative validity''
since $1-\epsilon$ is only a lower bound on the coverage probability.
However, the definition of conformal predictors can be slightly modified
(using randomization for treatment of borderline cases)
to achieve exact validity;
in practice, the difference between conformal predictors and their modified (``smoothed'')
version is negligible.
For details, see, e.g., \citet{vovk/etal:2005book}, p.~27.

\section{Main result}
\label{sec:result}

In this section we show that under the Gaussian model~(\ref{eq:model})
complemented by other natural (and standard) assumptions
CRR is asymptotically close to BRR,
and therefore is approximately conditionally valid and efficient.
On the other hand,
Proposition~\ref{prop:validity} guarantees the unconditional validity of CRR
under the IID assumption,
regardless of whether (\ref{eq:model}) holds.

In this section we assume an infinite sequence of observations
\ifCONF$(x_1,y_1),(x_2,y_2),\ldots$ \fi
\ifnotCONF$$(x_1,y_1),(x_2,y_2),\ldots$$\fi%
but consider only the first $n$ of them
and let $n\to\infty$.
We make both the IID assumption about the objects $x_1,x_2,\ldots$
(the objects are generated independently from the same distribution)
and the assumption (\ref{eq:model});
however, we relax the assumption that $w$ is distributed as $N(0,(\sigma^2/a)I)$.
These are all the assumptions used in our main result:
\begin{description}
\item[(A1)]
  The random objects $x_i\in\mathbb{R}^p$, $i=1,2,\ldots$, are IID.
\item[(A2)]
  The second-moment matrix $\Expect(x_1x'_1)$ of $x_1$ exists and is non-singular.
\item[(A3)]
  The random vector $w\in\mathbb{R}^p$
  is independent of $x_1,x_2,\ldots$\,.
\item[(A4)]
  The labels $y_1,y_2,\ldots$ are generated by
  $y_i=w\cdot x_i+\xi_i$,
  where $\xi_i$ are Gaussian noise variables distributed as $N(0,\sigma^2)$
  and independent between themselves, of the objects $x_i$,
  and of $w$.
\end{description}
Notice that the assumptions imply that the random observations $(x_i,y_i)$, $i=1,2,\ldots$,
are IID given $w$.
It will be clear from the proof
that the assumptions can be relaxed further
(but we have tried to make them as simple as possible).

\begin{theorem}\label{thm:main}
  Under the assumptions (A1)--(A4),
  the prediction sets output by CRR are intervals from some $n$ on almost surely,
  and the differences between the upper and lower ends of the prediction intervals
  for BRR and CRR are asymptotically Gaussian:
  \begin{align}
    \sqrt{n}
    (B^*-C^*)
    &\law
    N
    \left(
      0,
      \frac{\alpha(1-\alpha)}{f^2(\zeta_{\alpha})}
      -
      \sigma^2
      \mu'\Sigma^{-1}\mu
    \right),
    \label{eq:to-prove}\\
    \sqrt{n}
    (B_*-C_*)
    &\law
    N
    \left(
      0,
      \frac{\alpha(1-\alpha)}{f^2(\zeta_{\alpha})}
      -
      \sigma^2
      \mu'\Sigma^{-1}\mu
    \right),
    \label{eq:not-to-prove}
  \end{align}
  where $\alpha:=1-\epsilon/2$,
  $\zeta_{\alpha}:=z_{\epsilon/2}\sigma$ is the $\alpha$-quantile of $N(0,\sigma^2)$,
  $f$ is the density of $N(0,\sigma^2)$,
  $\mu:=\Expect(x_1)$ is the expectation of $x_1$,
  and $\Sigma:=\Expect(x_1x'_1)$ is the second-moment matrix of $x_1$.
\end{theorem}

The theorem will be proved in Section~\ref{sec:proof},
and in the rest of this section we will discuss it.
We can see from (\ref{eq:to-prove}) and (\ref{eq:not-to-prove})
that the symmetric difference
between the prediction intervals output by BRR and CRR
shrinks to 0 as $O(n^{-1/2})$ in Lebesgue measure with high probability.

Let us first see what the typical values of the standard deviation
(the square root of the variance)
in~(\ref{eq:to-prove}) and~(\ref{eq:not-to-prove}) are.
It is easy to check that the standard deviation is proportional to $\sigma$;
therefore, let us assume $\sigma=1$.
The second term in the variance does not affect it significantly
since $0\le\mu'\Sigma^{-1}\mu\le1$.
Indeed, denoting the covariance matrix of $x_1$ by $C$ and using the Sherman--Morrison formula
(see, e.g., \citealt{henderson/searle:1981}, (3)),
we have:
\begin{multline}\label{eq:second-term}
  \mu'\Sigma^{-1}\mu
  =
  \mu'(C+\mu\mu')^{-1}\mu
  =
  \mu'
  \left(
    C^{-1}
    -
    \frac{C^{-1}\mu\mu'C^{-1}}{1+\mu'C^{-1}\mu}
  \right)
  \mu\\
  =
  \mu'C^{-1}\mu
  -
  \frac{(\mu'C^{-1}\mu)^2}{1+\mu'C^{-1}\mu}
  =
  \frac{\mu'C^{-1}\mu}{1+\mu'C^{-1}\mu}
  \in
  [0,1]
\end{multline}
(we write $[0,1]$ rather than $(0,1)$ because $C$ is permitted to be singular:
see Appendix~\ref{app:computations} for details).
The first term, on the other hand,
can affect the variance more significantly,
and the significant dependence of the variance on $\epsilon$ is natural:
the accuracy obtained from the Gaussian model is better for small $\epsilon$
since it uses all data for estimating the end-points of the prediction interval
rather than relying, under the IID model, on the scarcer information provided by observations
in the tails of the distribution generating the labels.
Figure~\ref{fig:variance} illustrates the dependence
of the standard deviation of the asymptotic distribution on $\epsilon$.
The upper line in it corresponds to $\mu'\Sigma^{-1}\mu=0$
and the lower line corresponds to $\mu'\Sigma^{-1}\mu=1$.
The possible values for the standard deviation lie between the upper and lower lines.
The asymptotic behaviour of the standard deviation as $\epsilon\to0$ is given by
\begin{equation}\label{eq:std}
  \sqrt
  {
    \epsilon
    (1-\epsilon/2)
    \pi
    e^{z_{\epsilon/2}^2}
    -
    \theta
  }
  \sim
  \left(
    -\epsilon\ln\epsilon
  \right)^{-1/2}
\end{equation}
uniformly in $\theta\in[0,1]$.

\begin{figure}[tb]
  \begin{center}
    \includegraphics[width=0.48\textwidth]{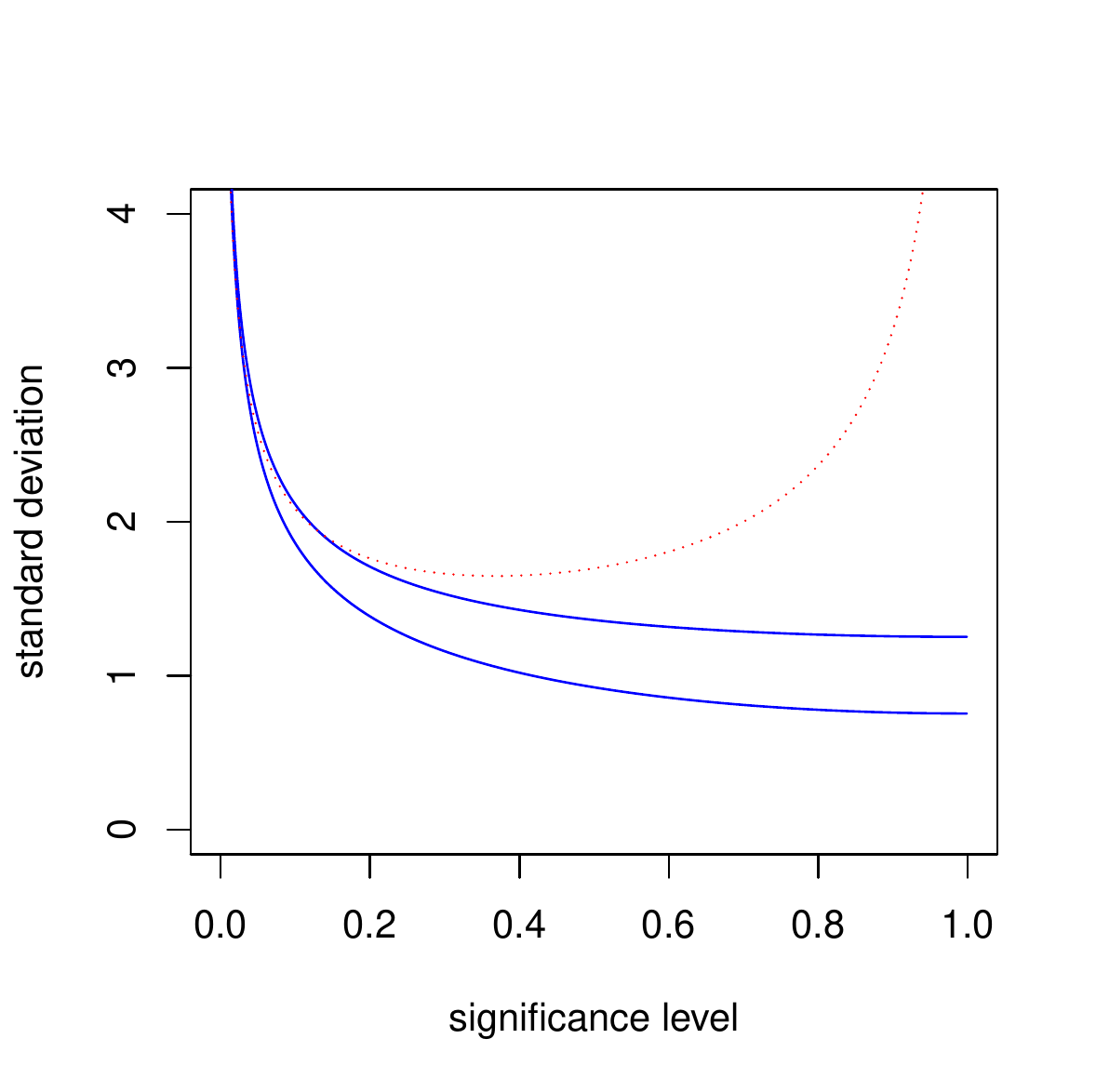}
    \includegraphics[width=0.48\textwidth]{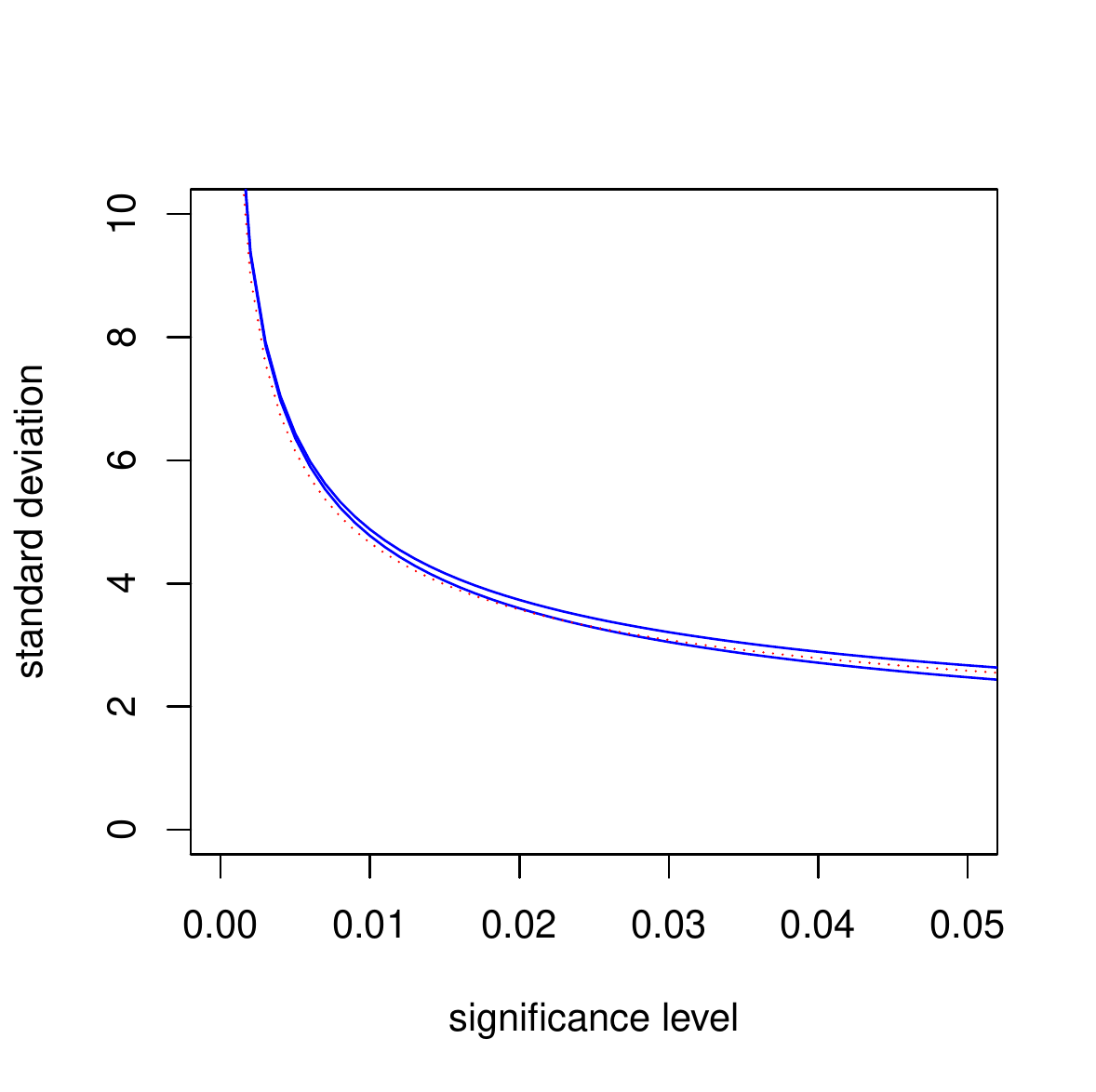}
  \end{center}
  \caption{The limits for the standard deviation
    in Theorem~\ref{thm:main}
    as a function of $\epsilon\in(0,1)$ (left) and $\epsilon\in(0,0.05]$ (right)
    shown as solid (blue) lines;
    the asymptotic expression in~\eqref{eq:std} shown as a dotted (red) line.
    In all cases $\sigma=1$.}
  \label{fig:variance}
\end{figure}

The assumptions~(A1)--(A4) do not involve $a$,
and Theorem~\ref{thm:main} continues to hold if we set $a:=0$;
this can be checked by going through the proof of Theorem~\ref{thm:main}
in Section~\ref{sec:proof}.
Theorem~\ref{thm:main} can thus also be considered
as an efficiency result about conformalizing
the standard non-Bayesian least squares procedure;
this procedure outputs precisely $(B_*,B^*)$ with $a:=0$
as its prediction intervals
(see, e.g., \citealt{seber/lee:2003}, p.~131).
The least squares procedure has guaranteed coverage probability
under weaker assumptions than BRR
(not requiring assumptions about $w$);
however, its validity is not conditional, similarly to CRR.

\section{Further details of CRR}
\label{sec:details}

By the definition of the CRR conformity measure,
we can rewrite the conformity scores in~(\ref{eq:alpha-1}) as
\begin{equation}\label{eq:alpha-2}
  \alpha^y_i
  :=
  \left|
    \left\{
      j=1,\ldots,n \mid r^y_j\ge r^y_i
    \right\}
  \right|
  \wedge
  \left|
    \left\{
      j=1,\ldots,n \mid r^y_j\le r^y_i
    \right\}
  \right|,
\end{equation}
where the vector of residuals $(r^y_1,\ldots,r^y_n)'$ is $(I_n-H_n)Y^y$,
$I_n$ is the unit $n\times n$ matrix,
$H_n:=X_n(X'_nX_n+aI)^{-1}X'_n$ is the hat matrix,
$X_n$ is the overall design matrix
(the $n\times p$ matrix whose $i$th row is $x'_i$, $i=1,\ldots,n$),
and $Y^y$ is the overall vector of labels
with the label of the test object set to $y$
(i.e., $Y^y$ is the vector in $\mathbb{R}^n$ whose $i$th element is $y^i$, $i=1,\ldots,n-1$,
and whose $n$th element is $y$).
If we modify the definition of CRR replacing (\ref{eq:alpha-2}) by $\alpha^y_i:=-r^y_i$,
we will obtain the definition of \emph{upper CRR};
and if we replace (\ref{eq:alpha-2}) by $\alpha^y_i:=r^y_i$,
we will obtain the definition of \emph{lower CRR}.
It is easy to see that the prediction set $\Gamma$ output by CRR
at significance level $\epsilon$
is the intersection of the prediction sets output by upper and lower CRR
at significance levels $\epsilon/2$.
We will concentrate on upper CRR in the rest of this paper:
lower CRR is analogous,
and CRR is determined by upper and lower CRR.

Let us represent the upper CRR prediction set in a more explicit form
(following \citealt{vovk/etal:2005book}, Section~2.3).
We are given the training sequence $(x_1,y_1),\ldots,(x_{n-1},y_{n-1})$
and a test object $x_n$;
let $y$ be a postulated label for $x_n$ and
$$
  Y^y
  :=
  (y_1,\ldots,y_{n-1},y)'
  =
  (y_1,\ldots,y_{n-1},0)'
  +
  y(0,\ldots,0,1)'
$$
be the vector of labels.
The vector of conformity scores is $-(I_n-H_n)Y^y=-A-yB$,
where
\begin{align*}
  A&:=(I_n-H_n)(y_1,\ldots,y_{n-1},0)',\\
  B&:=(I_n-H_n)(0,\ldots,0,1)'.
\end{align*}
The components of $A$ and $B$, respectively,
will be denoted by $a_1,\ldots,a_n$ and $b_1,\ldots,b_n$.

If we define
\begin{equation}\label{eq:S}
  S_i
  :=
  \left\{
    y
    \mid
    -a_i-b_iy \le -a_n-b_ny
  \right\},
\end{equation}
the definition of the p-values can be rewritten as
$$
  p^y
  :=
  \frac
  {\left|\left\{i=1,\ldots,n\mid y\in S_i\right\}\right|}
  {n};
$$
remember that the prediction set is defined by (\ref{eq:prediction-set}).
As shown (under a slightly different definition of $S_i$)
in \citet{vovk/etal:2005book}, pp.~30--34,
the prediction set can be computed efficiently,
in time $O(n\ln n)$ in the on-line mode.

\section{Proof of Theorem~\ref{thm:main}}
\label{sec:proof}

For concreteness, we concentrate on the convergence (\ref{eq:to-prove})
for the upper ends of the conformal and Bayesian prediction intervals.
We split the proof into a series of steps.

\subsection*{Regularizing the rays in upper CRR}

The upper CRR looks difficult to analyze in general,
since the sets (\ref{eq:S}) may be rays pointing in the opposite directions.
Fortunately, the awkward case $b_n\le b_i$ ($i<n$)
will be excluded for large $n$ under our assumptions
(see Lemma~\ref{lem:excluded} below).
The following lemma gives a simple sufficient condition for its absence.

\begin{lemma}\label{lem:diversity}
  Suppose that, for each $c\in\mathbb{R}^p\setminus\{0\}$,
  \begin{equation}\label{eq:diversity}
    (c\cdot x_n)^2
    <
    \sum_{i=1}^{n-1} (c\cdot x_i)^2
    +
    a\left\|c\right\|^2,
  \end{equation}
  where $\left\|\cdot\right\|$ stands for the Euclidean norm.
  Then $b_n>b_i$ for all $i=1,\ldots,n-1$.
\end{lemma}

\noindent
Intuitively, in the case of a small $a$,
(\ref{eq:diversity}) being violated for some $c\ne0$
means that all $x_1,\ldots,x_{n-1}$ lie approximately
in the same hyperplane, and $x_n$ is well outside it.
The condition~(\ref{eq:diversity}) can be expressed by saying that the matrix
$\sum_{i=1}^{n-1}x_ix'_i-x_nx'_n+aI$ is positive definite.

\medskip

\begin{proof}
  First we assume $a=0$ (so that ridge regression becomes least squares);
  an extension to $a\ge0$ will be easy.
  In this case $H_n$ is the projection matrix
  onto the column space $\mathcal{C}\subseteq\mathbb{R}^n$
  of the overall design matrix $X_n$
  and $I_n-H_n$ is the projection matrix
  onto the orthogonal complement $\mathcal{C}^{\perp}$ of $\mathcal{C}$.
  We can have $b_n\le b_i$ for $i<n$
  (or even $b_n^2\le b_1^2+\cdots+b_{n-1}^2$)
  only if the angle between $\mathcal{C}^{\perp}$ and the hyperplane $\mathbb{R}^{n-1}\times\{0\}$
  is $45^{\circ}$ or less;
  in other words, if the angle between $\mathcal{C}$ and that hyperplane is $45^{\circ}$ or more;
  in other words, if there is an element $(c\cdot x_1,\ldots,c\cdot x_n)'$
  of $\mathcal{C}$ such that its last coordinate is $c\cdot x_n=1$
  and its projection $(c\cdot x_1,\ldots,c\cdot x_{n-1})'$ onto the other coordinates
  has length at most 1.

  To reduce the case $a>0$ to $a=0$
  add the $p$ dummy objects $\sqrt{a}e_i\in\mathbb{R}^p$,
  $i=1,\ldots,p$,
  labelled by 0 at the beginning of the training sequence;
  here $e_1,\ldots,e_p$ is the standard basis of $\mathbb{R}^p$.
\end{proof}

\begin{lemma}\label{lem:excluded}
  The case $b_n\le b_i$ for $i<n$ is excluded from some $n$ on almost surely under (A1)--(A4).
\end{lemma}

\begin{proof}
  We will check that~(\ref{eq:diversity}) holds from some $n$ on.
  Let us set, without loss of generality, $a:=0$.
  Let $\Sigma_l := \frac{1}{l}\sum_{i=1}^lx_ix'_i$.
  Since $\lim_{l\to\infty}\Sigma_l = \Sigma$ a.s.,
  \begin{equation*}
    \left|
      \lambda_{\min}(\Sigma_l)
      -
      \lambda_{\min}(\Sigma)
    \right|
    \to
    0
    \quad
    (l\to\infty)
    \quad
    \text{a.s.},
  \end{equation*}
  where $\lambda_{\min}(\cdot)$ is the smallest eigenvalue of the given matrix.
  Since $\left\|x_n\right\|^2/n \to 0$ a.s.,
  \begin{multline*}     
    \frac{1}{n-1}
    \sum_{i=1}^{n-1}
    (c\cdot x_i)^2
    =
    c'\Sigma_{n-1}c
    \ge
    \lambda_{\min}(\Sigma_{n-1})
    \left\|c\right\|^2\\    
    >
    \frac12
    \lambda_{\min}(\Sigma)
    \left\|c\right\|^2
    >
    \frac{\left\|c\right\|^2\left\|x_n\right\|^2}{n-1}
    \ge
    \frac{(c\cdot x_n)^2}{n-1}
  \end{multline*}      
  for all $c\ne0$ from some $n$ on.
\end{proof}

\subsection*{Simplified upper CRR}

Let us now find the upper CRR prediction set
under the assumption that $b_n>b_i$ for all $i<n$
(cf.\ Lemmas~\ref{lem:diversity} and~\ref{lem:excluded} above).
In this case each set (\ref{eq:S}) is
$$
  S_i=(-\infty,t_i],
  \qquad
  \text{where }
  t_i
  :=
  \frac{a_i-a_n}{b_n-b_i},
$$
except for $S_n:=\mathbb{R}$;
notice that only $t_1,\ldots,t_{n-1}$ are defined.
The p-value $p^y$ for any potential label $y$ of $x_n$ is
\begin{equation*}
  p^y
  =
  \frac{\left|\{i=1,\ldots,n\mid y\in S_i\}\right|}{n}
  =
  \frac{\left|\{i=1,\ldots,n-1\mid t_i\ge y\}\right|+1}{n}.
\end{equation*}
Therefore, the upper CRR prediction set
at significance level $\epsilon/2$
is the ray
$$
  (-\infty,t_{(k_n)}],
$$
where $k_n:=\lceil(1-\epsilon/2)n\rceil$
and $t_{(k)}=t_{k:(n-1)}$ stands, as usual, for the $k$th order statistic
of $t_1,\ldots,t_{n-1}$.

\subsection*{Proof proper}

As before,
$X$ stands for the design matrix $X_{n-1}$
based on the first $n-1$ observations.
A simple but tedious computation (see Appendix~\ref{app:computations}) gives
\begin{equation}\label{eq:t}
  t_i
  =
  \frac{a_i-a_n}{b_n-b_i}
  =
  \hat y_n
  +
  (y_i - \hat y_i)
  \frac{1+g_n}{1+g_i},
\end{equation}
where $g_i:=x'_i(X'X+aI)^{-1}x_n$ (cf.\ (\ref{eq:g-n})).
The first term in (\ref{eq:t}) is the centre
of the Bayesian prediction interval~(\ref{eq:Bayes});
it does not depend on $i$.
We can see that
\begin{equation}\label{eq:B-C}
  B^*-C^*
  =
  (1+g_n)
  \left(
    z_{\epsilon/2}\sigma
    -
    V_{(k_n)}
  \right),
\end{equation}
where $V_{(k_n)}$ is the $k_n$th order statistic in the series
\begin{equation}\label{eq:V}
  V_i
  :=
  \frac{r_i}{1+g_i}
\end{equation}
of residuals $r_i:=y_i-\hat y_i$ adjusted by dividing by $1+g_i$.
The behaviour of the order statistics of residuals
is well studied: see, e.g., the theorem in \citet{carroll:1978}.
The presence of $1+g_i$ complicates the situation,
and so we first show that $g_i$ is small with high probability.

\begin{lemma}\label{lem:hin}
  Let $\eta_1,\eta_2,\ldots$ be a sequence of IID random variables with a finite second moment.
  Then $\max_{i=1,\ldots,n}\left|\eta_i\right|=o(n^{1/2})$
  in probability (and even almost surely) as $n\to\infty$.
\end{lemma}
\begin{proof}
  By the strong law of large numbers the sequence
  $\frac1n\sum_{i=1}^n\eta_i^2$ converges a.s.\ as $n\to\infty$,
  and so $\eta_n^2/n\to0$ a.s.
  This implies that
  $\max_{i=1,\ldots,n}\left|\eta_i\right|=o(n^{1/2})$ a.s.
\end{proof}

\begin{corollary}\label{cor:g_i}
  Under the conditions of the theorem, $\max_{i=1,\ldots,n}\left|g_i\right|=o(n^{-1/2})$ in probability.
\end{corollary}
\begin{proof}
  Similarly to the proof of Lemma~\ref{lem:excluded},
  we have, for almost all sequences $x_1,x_2,\ldots$,
  $$
    \max_{i=1,\ldots,n}
    \left|g_i\right|
    \le
    \frac
    {
      \left\|x_n\right\|
      \max_{i=1,\ldots,n}
      \left\|x_i\right\|
    }
    {
      \lambda_{\min}(X'X+aI)
    }
    <
    2
    \frac
    {
      \left\|x_n\right\|
      \max_{i=1,\ldots,n}
      \left\|x_i\right\|
    }
    {
      (n-1)\lambda_{\min}(\Sigma)
    }
  $$
  from some $n$ on.
  It remains to combine this with Lemma~\ref{lem:hin}
  and the fact that, by Assumption~(A1),
  $\left\|x_n\right\|$ is bounded by a constant with high probability.
\end{proof}

\begin{corollary}\label{cor:reduction}
  Under the conditions of the theorem,
  $n^{1/2}\left(r_{(k_n)}-V_{(k_n)}\right)\to0$ in probability.
\end{corollary}
\begin{proof}
  Suppose that, on the contrary, there are $\epsilon>0$ and $\delta>0$
  such that $n^{1/2}\left|r_{(k_n)}-V_{(k_n)}\right|>\epsilon$
  with probability at least $\delta$ for infinitely many $n$.
  Fix such $\epsilon$ and $\delta$.
  Suppose, for concreteness,
  that, with probability at least $\delta$ for infinitely many $n$,
  we have $n^{1/2}\left(r_{(k_n)}-V_{(k_n)}\right)>\epsilon$,
  i.e.,
  $V_{(k_n)}<r_{(k_n)}-\epsilon n^{-1/2}$.
  The last inequality implies that $V_i<r_{(k_n)}-\epsilon n^{-1/2}$
  for at least $k_n$ values of $i$.
  By the definition (\ref{eq:V}) of $V_i$
  this in turn implies that
  $r_i<r_{(k_n)}-\epsilon n^{-1/2}+g_i r_{(k_n)}$
  for at least $k_n$ values of $i$.
  By Corollary~\ref{cor:g_i}, however, the last addend is less than $\epsilon n^{-1/2}$
  with probability at least $1-\delta$ from some $n$ on
  (the fact that $r_{(k_n)}$ is bounded with high probability follows, e.g.,
  from Lemma~\ref{lem:Carroll} below).
  This implies $r_{(k_n)}<r_{(k_n)}$ with positive probability from some $n$ on,
  and this contradiction completes the proof.
\end{proof}

The last (and most important) component of the proof
is the following version of the theorem in \citet{carroll:1978},
itself a version of the famous Bahadur representation theorem \citep{bahadur:1966}.

\begin{lemma}[\citealt{carroll:1978}, theorem]\label{lem:Carroll}
  Under the conditions of Theorem~\ref{thm:main},
  \begin{equation}\label{eq:Carroll}
    n^{1/2}
    \left|
      \left(r_{(k_n)} - \zeta_{\alpha}\right)
      -
      \frac{\alpha - F_n(\zeta_{\alpha})}{f(\zeta_{\alpha})}
      +
      \mu'
      (\hat w_n - w)
    \right|
    \to
    0
    \qquad
    \text{a.s.},
  \end{equation}
  where $F_n$ is the empirical distribution function of the noise $\xi_1,\ldots,\xi_{n-1}$
  and $\hat w_n:=(X'X+aI)^{-1}X'Y$ is the ridge regression estimate of $w$.
\end{lemma}

For details of the proof (under our assumptions),
see Appendix~\ref{app:proof-of-lemma}.

By (\ref{eq:B-C}), Corollary~\ref{cor:g_i},
and Slutsky's lemma (see, e.g., \citealt{vandervaart:1998}, Lemma~2.8),
it suffices to prove (\ref{eq:to-prove})
with the left-hand side replaced by
$n^{1/2}(V_{(k_n)} - z_{\epsilon/2}\sigma)$.
Moreover,
by Corollary~\ref{cor:reduction} and Slutsky's lemma,
it suffices to prove~(\ref{eq:to-prove})
with the left-hand side replaced by
$n^{1/2}(r_{(k_n)} - z_{\epsilon/2}\sigma)$;
this is what we will do.

Lemma~\ref{lem:Carroll} holds in the situation
where $w$ is a constant vector
(the distribution of $w$ is allowed to be degenerate).
Let $R$ be a Borel set in $(\mathbb{R}^p)^{\infty}$
such that (\ref{eq:Carroll}) holds for all $(x_1,x_2,\ldots)\in R$,
where the ``a.s.''\ is now interpreted
as ``for almost all sequences $(\xi_1,\xi_2,\ldots)$''.
By Lebesgue's dominated convergence theorem,
it suffices to prove~(\ref{eq:to-prove})
with the left-hand side replaced by
$n^{1/2}(r_{(k_n)} - z_{\epsilon/2}\sigma)$
for a fixed $w$ and a fixed sequence $(x_1,x_2,\ldots)\in R$.
Therefore, we fix $w$ and $(x_1,x_2,\ldots)\in R$;
the only remaining source of randomness is $(\xi_1,\xi_2,\ldots)$.
Finally,
by the definition of the set $R$,
it suffices to prove (\ref{eq:to-prove})
with the left-hand side replaced by
\begin{equation}\label{eq:to-analyze}
  n^{1/2}
  \frac{\alpha - F_n(\zeta_{\alpha})}{f(\zeta_{\alpha})}
  -
  n^{1/2}
  \mu'
  (\hat w_n - w).
\end{equation}
Without loss of generality we will assume
that $\frac1nX'_nX_n\to\Sigma$ as $n\to\infty$
(this extra assumption about $R$ will ensure that Lindeberg's condition is satisfied below).

Since $\Expect(\alpha - F_n(\zeta_{\alpha})) = 0$ and
$$
  \var\left(\alpha-F_n(\zeta_{\alpha})\right)
  =
  \frac{F(\zeta_{\alpha})(1-F(\zeta_{\alpha}))}{n-1} = \frac{\alpha(1-\alpha)}{n-1},
$$
where $F$ is the distribution function of $N(0,\sigma^2)$,
we have
\[
  n^{1/2}
  \frac{\alpha-F_n(\zeta_{\alpha})}{f(\zeta_{\alpha})}
  \law
  N
  \left(
    0,\frac{\alpha(1-\alpha)}{f^2(\zeta_{\alpha})}
  \right)
  \quad
  (n\to\infty)
\]
by the central limit theorem (in its simplest form).

Since $\hat w_n=(X'X+aI)^{-1}X'Y$ is the ridge regression estimate,
\begin{align}
  \Expect(\hat w_n - w)
  &=
  -a(X'X+aI)^{-1}w
  =:
  \Delta_n,
  \label{eq:Delta}\\
  \var(\hat w_n)
  &=
  \sigma^2
  (X'X+aI)^{-1}X'X(X'X+aI)^{-1}
  =:
  \Omega_n.
  \label{eq:Omega}
\end{align}
Furthermore, for $n\to\infty$
\begin{align*}
  n^{1/2}\Delta_n
  &=
  -n^{-1/2}a\left(\frac{X'X}{n}+\frac{aI}{n}\right)^{-1}w
  \sim
  -n^{-1/2}a\Sigma^{-1}w
  \to
  0,\\
  n\Omega_n
  &=
  \sigma^2
  \left(\frac{X'X}{n}+\frac{aI}{n}\right)^{-1}
  \frac{X'X}{n}
  \left(\frac{X'X}{n}+\frac{aI}{n}\right)^{-1}
  \to
  \sigma^2\Sigma^{-1}.
\end{align*}
This gives
\[
  n^{1/2}\mu'(\hat w_n-w)
  \law
  N
  \left(
    0,\sigma^2\mu'\Sigma^{-1}\mu
  \right)
  \quad
  (n\to\infty)
\]
(the asymptotic, and even exact,
normality is obvious from the formula for $\hat w_n$).

Let us now calculate the covariance
between the two addends in \eqref{eq:to-analyze}:
\begin{align*}       
  \cov &      
  \left(
    n^{1/2}\frac{\alpha-F_n(\zeta_{\alpha})}{f(\zeta_{\alpha})},
    -n^{1/2}\mu'(\hat w_n-w)
  \right)\\&   
  =
  \frac{n}{f(\zeta_{\alpha})}
  \cov
  \left(
    F_n(\zeta_{\alpha})-\alpha,
    \mu'(\hat w_n - w)
  \right)\\
  &=         
  \frac{n}{(n-1)f(\zeta_{\alpha})}
  \sum_{i=1}^{n-1}
  \cov
  \left(
    1_{\{\xi_i\le\zeta_{\alpha}\}}-\alpha,
    \mu'(\hat w_n - w)
  \right)\\
  &=         
  \frac{n}{(n-1)f(\zeta_{\alpha})}
  \sum_{i=1}^{n-1}
  \Expect
  \Bigl(
    \left(
      1_{\{\xi_i\le\zeta_{\alpha}\}}-\alpha
    \right)
    \mu'(X'X+aI)^{-1}X'\xi
  \Bigr),
\end{align*}         
where $\xi = (\xi_1,\ldots,\xi_{n-1})'$ and the last equality uses the decomposition
$
  \hat w_n - w
  =
  \Delta_n + (X'X+aI)^{-1}X'\xi
$
with the second addend having zero expected value.
Since
\begin{equation*}
  \Expect
  1_{\{\xi_i\le\zeta_{\alpha}\}}
  \mu'(X'X+aI)^{-1}X'\xi
  =
  \sum_{j=1}^{n-1}
  \Expect
  1_{\{\xi_i\le\zeta_{\alpha}\}}
  A_j\xi_j
  =
  \mu_{\alpha} A_i, 
\end{equation*}
where $A_j:=\mu'(X'X+aI)^{-1}x_j$, $j = 1,\ldots,n-1$,
$
  \mu_{\alpha}
  :=
  \Expect
  1_{\{\xi_i\le\zeta_{\alpha}\}}\xi_i
  =
  \int_{-\infty}^{\zeta_{\alpha}}xf(x)dx
$.
An easy computation gives
$\mu_{\alpha}=-\sigma^2f(\zeta_{\alpha})$,
and so we have
\begin{multline*}
  \cov
  \left(
    n^{1/2}
    \frac{\alpha-F_n(\zeta_{\alpha})}{f(\zeta_{\alpha})},
    -n^{1/2}
    \mu'(\hat w_n-w)
  \right)
  =
  \frac{n}{(n-1)f(\zeta_{\alpha})}
  \sum_{i=1}^{n-1}
  \mu_{\alpha} A_i\\
  =
  - \sigma^2 \frac{n}{(n-1)}
  \sum_{i=1}^{n-1}
  A_i
  =
  -\sigma^2
  \mu'
  \left(
    \frac1n X'X+\frac{a}{n}I
  \right)^{-1}
  \bar{x}
  \to
  -\sigma^2
  \mu'\Sigma^{-1}\mu
\end{multline*}
as $n\to\infty$,
where $\bar{x}$ is the arithmetic mean of $x_1,\ldots,x_{n-1}$.
Finally, this implies that (\ref{eq:to-analyze}) converges in law to
\begin{equation*}
  N
  \left(
    0,
    \frac{\alpha(1-\alpha)}{f^2(\zeta_{\alpha})}
    +
    \sigma^2
    \mu'\Sigma^{-1}\mu
    -
    2\sigma^2
    \mu'\Sigma^{-1}\mu
  \right)
  =
  N
  \left(
    0,
    \frac{\alpha(1-\alpha)}{f^2(\zeta_{\alpha})}
    -
    \sigma^2
    \mu'\Sigma^{-1}\mu
  \right);
\end{equation*}
the asymptotic normality of~(\ref{eq:to-analyze}) follows from the central limit theorem 
with Lindeberg's condition,
which holds since (\ref{eq:to-analyze})
is a linear combination of the noise random variables
$\xi_1,\ldots,\xi_{n-1}$
with coefficients whose maximum is $o(1)$ as $n\to\infty$
(this uses the assumption $\frac1nX'_nX_n\to\Sigma$
made earlier).

A more intuitive (but not necessarily simpler)
proof can be obtained by noticing that $\hat w_n-w$ and the residuals
are asymptotically (precisely when $a=0$) independent.

\section{Conclusion}

The results of this paper are asymptotic;
it would be very interesting to obtain their non-asymptotic counterparts.
In non-asymptotic settings, however, it is not always true
that conformalized ridge regression loses little in efficiency
as compared with the Bayesian prediction interval;
this is illustrated in \citet{vovk/etal:2005book}, Section~8.5,
and illustrated and explained in \citet{vovk/etal:2009AOS}.
The main difference is that CRR and Bayesian predictor
start producing informative predictions after seeing a different number of observations.
CRR, like any other conformal predictor
(or any other method whose validity depends only on the IID assumption),
starts producing informative predictions only after the number of observations exceeds
the inverse significance level $1/\epsilon$.
After this theoretical lower bound is exceeded, however,
the difference between CRR and Bayesian predictions quickly becomes very small.

Another interesting direction of further research is to extend
our results to kernel ridge regression.

\ifnotCONF
  \subsection*{Acknowledgements}

  We are grateful
  to Albert Shiryaev for inviting us in September 2013
  to Kolmogorov's dacha in Komarovka,
  where this project was conceived,
  and to Glenn Shafer for his advice about terminology.
  This work was supported in part by EPSRC (grant EP/K033344/1).
\fi

\appendix
\section{Various computations}
\label{app:computations}

For the reader's convenience,
this appendix provides details of various routine calculations.

\subsection*{A singular $C$ in (\ref{eq:second-term})}

Apply~(\ref{eq:second-term}) to
$\Sigma_{\epsilon}:=\Sigma+\epsilon I$ and $C_{\epsilon}:=C+\epsilon I$, where $\epsilon>0$,
in place of $\Sigma$ and $C$, respectively,
and let $\epsilon\to0$.

\subsection*{Computing $t_i$ for simplified upper CRR}

In addition to the notation $X$ for the design matrix $X_{n-1}$
based on the first $n-1$ observations,
we will use the notation $H$ for the hat matrix $X(X'X+aI)^{-1}X'$
based on the first $n-1$ observations
and $\bar H$ for the hat matrix $X_n(X'_nX_n+aI)^{-1}X'_n$
based on the first $n$ observations;
the elements of $H$ will be denoted as $h_{i,j}$
and the elements of $\bar H$ as $\bar h_{i,j}$;
as always, $h_i$ stands for the diagonal element $h_{i,i}$.
To compute $t_i$ we will use the formulas~(2.18) in \citet{chatterjee/hadi:1988}.

  Since $B$ is the last column of $I_n-H_n$ and
  $$
    \bar h_{n,n}
    =
    \frac{x'_n(X'X+aI)^{-1}x_n}{1+x'_n(X'X+aI)^{-1}x_n},
  $$
  we have
  \begin{align*}
    b_n
    &=
    1 - \frac{x'_n(X'X+aI)^{-1}x_n}{1+x'_n(X'X+aI)^{-1}x_n},\\
    b_i
    &=
    \frac{-x'_n(X'X+aI)^{-1}x_i}{1+x'_n(X'X+aI)^{-1}x_n}.
  \end{align*}
  Therefore,
  $$
    b_n-b_i
    =
    \frac{1+x'_n(X'X+aI)^{-1}x_i}{1+x'_n(X'X+aI)^{-1}x_n}.
  $$
  Next, letting $\hat y$ stand for the predictions
  computed from the first $n-1$ observations,
  \begin{align*}
    a_i
    &=
    \sum_{j=1,\ldots,n-1:j\ne i}
    (-\bar h_{i,j}y_j)
    +
    (1-\bar h_{i,i})y_i\notag\\
    &=
    y_i
    -
    \sum_{j=1}^{n-1}
    \bar h_{i,j}y_j\\ 
    &=
    y_i
    -
    \sum_{j=1}^{n-1}
    h_{i,j}y_j
    +
    \sum_{j=1}^{n-1}
    \frac{x'_i(X'X+aI)^{-1}x_n x'_n(X'X+aI)^{-1}x_j}{1+x'_n(X'X+aI)^{-1}x_n}
    y_j\notag\\
    &=
    y_i
    -
    \hat y_i
    +
    \frac{x'_i(X'X+aI)^{-1}x_n x'_n(X'X+aI)^{-1}X'Y}{1+x'_n(X'X+aI)^{-1}x_n}\notag\\
    &=
    y_i
    -
    \hat y_i
    +
    \frac{x'_i(X'X+aI)^{-1}x_n \hat y_n}{1+x'_n(X'X+aI)^{-1}x_n}\notag
  \end{align*}
  for $i<n$, and
  \begin{align*}
    a_n
    &=
    \sum_{j<n}
    (-\bar h_{n,j}y_j)
    =
    -\sum_{j=1}^{n-1}
    \frac{x'_j(X'X+aI)^{-1}x_n}{1+x'_n(X'X+aI)^{-1}x_n}
    y_j\\
    &=
    -\frac{Y'X(X'X+aI)^{-1}x_n}{1+x'_n(X'X+aI)^{-1}x_n}.
  \end{align*}
  Therefore,
  \begin{equation*}
    a_i-a_n
    =
    y_i - \hat y_i
    +
    \frac{1+x'_i(X'X+aI)^{-1}x_n}{1+x'_n(X'X+aI)^{-1}x_n}
    \hat y_n.
  \end{equation*}
  This gives
  $$
    t_i
    =
    (y_i - \hat y_i)
    \frac{1+x'_n(X'X+aI)^{-1}x_n}{1+x'_i(X'X+aI)^{-1}x_n}
    +
    \hat y_n,
  $$
  i.e., (\ref{eq:t}).

\subsection*{Expressing $\mu_{\alpha}$ via $\zeta_{\alpha}$}

First we use the substitution $y:=x^2/2\sigma^2$ to obtain
\begin{equation}\label{eq:1}
  \frac{1}{\sqrt{2\pi}\sigma}
  \int_0^{\zeta_{\alpha}}
  e^{-x^2/2\sigma^2}
  x dx
  =
  \frac{\sigma}{\sqrt{2\pi}}
  \int_0^{\zeta_{\alpha}^2/2\sigma^2}
  e^{y}
  dy
  =
  \frac{\sigma}{\sqrt{2\pi}}
  \left(
    1 - e^{-\zeta_{\alpha}^2/2\sigma^2}
  \right).
\end{equation}
Replacing $\zeta_{\alpha}$ by $\infty$,
\begin{equation}\label{eq:2}
  \frac{1}{\sqrt{2\pi}\sigma}
  \int_0^{\infty}
  e^{-x^2/2\sigma^2}
  x dx
  =
  \frac{\sigma}{\sqrt{2\pi}}.
\end{equation}
Finally, subtracting (\ref{eq:2}) from (\ref{eq:1}) gives
\begin{equation*}
  \mu_{\alpha}
  =
  \frac{1}{\sqrt{2\pi}\sigma}
  \int_{-\infty}^{\zeta_{\alpha}}
  e^{-x^2/2\sigma^2}
  x dx
  =
  -\frac{\sigma}{\sqrt{2\pi}}
  e^{-\zeta_{\alpha}^2/2\sigma^2}
  =
  -\sigma^2
  f(\zeta_{\alpha}).
\end{equation*}

\section{Proof of Lemma~\ref{lem:Carroll}}
\label{app:proof-of-lemma}

  The proof is modelled on the proof of the theorem
  in Carroll's technical report \citep{carroll:1978}
  and on Section~2 of~\citet{bahadur:1966}.
  We cannot use the result of \citet{carroll:1978}
  since our conditions are somewhat different.
  Following \citet{carroll:1978},
  we only consider the case of simple linear regression ($p=1$).
  We will prove that (\ref{eq:Carroll}) holds for all $w$,
  so that $w$ will be a constant vector in $\mathbb{R}^p$
  throughout the proof.

  We start from the speed of convergence in the ridge regression estimate
  of regression weights.
  Let $a_n:=n^{-1/2}\ln n$.
  \begin{lemma}\label{lem:RR}
    Under our conditions,
    $\left|\hat w_n-w\right|=o(a_n)$ a.s.
  \end{lemma}

  \begin{proof}
    This follows immediately from (\ref{eq:Delta}) and (\ref{eq:Omega}).
  \end{proof}

  The proof uses the following random variables:
  \begin{align*}
    G_n(x)
    &:=
    n^{-1}
    \sum_{i=1}^n
    \Bigl(
      1_{\{r_i\le x\}}
      -
      1_{\{\xi_i\le\zeta_{\alpha}\}}
      -
      F\bigl(x+x_i(\hat w_n - w)\bigr)
      +
      F(\zeta_{\alpha})
    \Bigr),\\
    H_n
    &:=
    n^{1/2}
    \sup_{x\in J_n}
    \left|
      G_n(x)
    \right|,
  \end{align*}
  where $J_n:=[\zeta_{\alpha}-a_n,\zeta_{\alpha}+a_n]$ and
  \begin{equation}\label{eq:W}
    W_n(s,t)
    =
    n^{-1/2}
    \sum_{i=1}^n
    \Bigl(
      1_{\{\xi_i\le\zeta_{\alpha}+a_ns+a_ntx_i\}}
      -
      1_{\{\xi_i\le\zeta_{\alpha}\}}
      -
      F(\zeta_{\alpha}+a_ns+a_ntx_i)
      +
      F(\zeta_{\alpha})
    \Bigr).
  \end{equation}

  \begin{lemma}\label{lem:1}
    Under our conditions,
    \begin{equation}\label{eq:WW}
      \sup
      \left\{
        \left|W_n(s,t)\right|\mid s,t\in[0,1]
      \right\}
      \to
      0
      \quad
      \text{a.s.}
    \end{equation}
    and, therefore,
    $H_n\to0$ a.s.
  \end{lemma}

  \begin{proof}
    Since $r_i=\xi_i-x_i(\hat w_n-w)$ and $\hat w_n-w=o(a_n)$ a.s.,
    it is indeed true that (\ref{eq:WW}) implies
    $H_n\to0$ a.s.;
    therefore, we will only prove~(\ref{eq:WW}).
    Let $b_n\sim\ln^2n$ be a sequence of positive integers.
    It suffices to consider only $s$ and $t$ of the form $\eta_{r,n}:=r/b_n$ for $r=0,\ldots,b_n$.
    To see this, apply Taylor's expansion:
    if
    $\left|s-\eta_{r,n}\right|\le b_n^{-1}$
    and
    $\left|t-\eta_{p,n}\right|\le b_n^{-1}$,
    then
    \begin{multline*}
      \left|
        n^{-1}
        \sum_{i=1}^n
        \left(
          F(\zeta_{\alpha}+sa_n+ta_nx_i)
          -	
          F(\zeta_{\alpha}+\eta_{r,n}a_n+\eta_{p,n}a_nx_i)
        \right)
      \right|\\
      \le
      n^{-1}
      \sum_{i=1}^n
      f(\zeta^*)
      a_n b_n^{-1} (1+\left|x_i\right|)
      =
      O(a_n b_n^{-1})
      =
      o(n^{-1/2})
      \quad
      \text{a.s.}
    \end{multline*}
    for some $\zeta^*$
    (we have used the integrability of $x_1$).

    For fixed $s$ and $t$ we can apply Bernstein's inequality
    (see, e.g., \citealt{gyorfi/etal:2002}, Lemma~A.2).
    Let us fix a sequence $x_1,x_2,\ldots$
    such that $\frac1n\sum_{i=1}^nx_i\to\mu$
    (which happens with probability one under our conditions
    for some $\mu$, namely for $\mu:=\Expect(x_1)$).
    The cumulative variance (conditional on $x_1,x_2,\ldots$)
    of the addends in~(\ref{eq:W}) does not exceed
    $$
      \sum_{i=1}^n
      (a_ns+a_ntx_i)
      =
      O(na_n)
    $$
    a.s.\ (this again uses the integrability of $x_1$);
    therefore, for any $\epsilon>0$,
    $$
      \Prob
      \left\{
        \left|W_n(s,t)\right| > \epsilon
      \right\}
      \le
      c_0
      \exp
      \left(
        -c_1 n^{1/4}
      \right)
    $$
    from some $n$ on,
    where $c_0$ and $c_1$ are constants depending on $\epsilon$.
    The probability that $\left|W_n(s,t)\right|>\epsilon$
    for some $n\ge N$ and some $s,t$ of the form $\eta_{r,n}$
    does not exceed
    $$
      \sum_{n=N}^{\infty}
      b_n^2
      c_0
      \exp
      \left(
        -c_1 n^{1/4}
      \right)
      \to
      0
      \quad
      (N\to\infty)
      \quad
      \text{a.s.}
    $$
    This completes the proof of the lemma.
  \end{proof}

  Remember that $k_n=\lceil\alpha n\rceil$.
  \begin{lemma}\label{lem:2}
    From some $n$ on, $r_{(k_n)}\in J_n$ a.s.
  \end{lemma}

  \begin{proof}
    We will only show that $r_{(k_n)}\le\zeta_{\alpha}+a_n$
    from some $n$ on a.s.
    Since
    $$
      \Prob
      \left\{
        r_{(k_n)} > \zeta_{\alpha} + a_n
      \right\}
      \le
      \Prob
      \left\{
        \sum_{i=1}^n
	1_{\{\xi_i \le \zeta_{\alpha} + a_n + x_i(\hat w_n-w)\}}
	\le
	k_n
      \right\}.
    $$
    By Lemma~\ref{lem:RR},
    it suffices to show the existence of an $\epsilon>0$
    for which $Q_N(\epsilon)\to0$ as $N\to\infty$,
    where
    \begin{multline*}
      Q_N(\epsilon)
      :=
      \Prob
      \left\{
        \sum_{i=1}^n
	1_{\{\xi_i \le \zeta_{\alpha} + a_n + t a_n x_i\}}
	\le
	k_n
	\text{ for some $t\in[0,\epsilon]$ and $n\ge N$}
      \right\}\\
      =
      \Prob
      \biggl\{
        F_n(\zeta_{\alpha}+a_n)
	\le
	k_n / n
	+
	n^{-1}
	\sum_{i=1}^n
	\Bigl(
          F(\zeta_{\alpha}+a_n)
	  -
          F(\zeta_{\alpha}+a_n+ta_nx_i)
	\Bigr)\\
	-
	n^{-1/2}
	\left(
	  W_n(1,t) - W_n(1,0)
	\right)
	\text{ for some $t\in[0,\epsilon]$ and $n\ge N$}
      \biggr\}.
    \end{multline*}
    Using Lemma~\ref{lem:1} and the fact that
    \begin{align*}
      n^{-1}
      \sum_{i=1}^n
      &
      \Bigl(
        F(\zeta_{\alpha}+a_n)
	-
        F(\zeta_{\alpha}+a_n+ta_nx_i)
      \Bigr)\\
      &=
      -n^{-1}
      \sum_{i=1}^n
      f(\zeta_{\alpha}+a_n)
      ta_nx_i
      +
      O
      \left(
        n^{-1}
        \sum_{i=1}^n
        t^2a_n^2x_i^2
      \right)\\
      &=
      -n^{-1}
      \sum_{i=1}^n
      f(\zeta_{\alpha})
      ta_nx_i
      +
      O
      \left(
        n^{-1}
        \sum_{i=1}^n
        ta_n^2x_i
      \right)
      +
      O(a_n^2)\\
      &=
      -f(\zeta_{\alpha})
      ta_n\mu
      +
      O
      \left(
        (\ln\ln n)^{1/2}
        n^{-1/2}
        a_n
      \right)
      +
      O(a_n^2)\\
      &=
      -f(\zeta_{\alpha})
      ta_n\mu
      +
      o(n^{-1/2})
      \quad
      \text{a.s.}
    \end{align*}
    (where $\mu:=\Expect(x_1)$),
    we obtain
    \begin{multline*}      
      Q_N(\epsilon)
      =
      \Prob
      \Bigl\{
        F_n(\zeta_{\alpha}+a_n)
	\le
	\alpha
	-
        t a_n \mu f(\zeta_{\alpha})
	+
	o(n^{-1/2})\\      
	\text{ for some $t\in[0,\epsilon]$ and $n\ge N$}
      \Bigr\}.
    \end{multline*}      
    By Hoeffding's inequality
    (see, e.g., \citealt{gyorfi/etal:2002}, Lemma~A.3),
    when $\delta>0$ is sufficiently small,
    \begin{equation*}
      \Prob
      \left\{
        F_n(\zeta_{\alpha}+a_n)
	\le
	\alpha
	+
        \delta a_n
      \right\}
      \le
      \exp
      \left(
        -c na_n^2
      \right)
      =
      n^{-c\ln n}
    \end{equation*}
    for some constant $c>0$.
    This implies that indeed $Q_N(\epsilon)\to0$ as $N\to\infty$.
  \end{proof}

  Now we can finish the proof of Lemma~\ref{lem:Carroll}.
  Let $E_n$ be the empirical distribution function of $r_i$.
  Lemma~\ref{lem:RR} and the second order Taylor expansion imply
  \begin{multline}\label{eq:r}
    G_n(r_{(k_n)})
    =
    E_n(r_{(k_n)}) - F_n(\zeta_{\alpha})\\
    -
    n^{-1}
    \sum_{i=1}^n
    \Bigl(
      F(r_{(k_n)})
      +
      f(r_{(k_n)})
      x_i (\hat w_n-w)
      -
      F(\zeta_{\alpha})
    \Bigr)
    +
    O\left(\frac1n\sum_{i=1}^nx_i^2\right)
    o(a_n^2)\\
    =
    E_n(r_{(k_n)}) - F_n(\zeta_{\alpha})
    -
    F(r_{(k_n)})
    +
    F(\zeta_{\alpha})
    +
    n^{-1}
    \sum_{i=1}^n
    f(r_{(k_n)})
    x_i (\hat w_n-w)\\   
    +
    o(n^{-1/2})
    \quad
    \text{a.s.}
  \end{multline}
  Similarly,
  \begin{multline}     
  \label{eq:zeta}
    G_n(\zeta_{\alpha})
    =
    E_n(\zeta_{\alpha}) - F_n(\zeta_{\alpha})
    -
    F(\zeta_{\alpha})
    +
    F(\zeta_{\alpha})
    +
    n^{-1}
    \sum_{i=1}^n
    f(\zeta_{\alpha})
    x_i (\hat w_n-w)\\    
    +
    o(n^{-1/2})
    \quad
    \text{a.s.}
  \end{multline}          
  Subtracting (\ref{eq:r}) from (\ref{eq:zeta})
  and using Lemmas~\ref{lem:1} and~\ref{lem:2}
  and the fact that $E_n(r_{(k_n)})=k_n/n$,
  we obtain
  \begin{multline}\label{eq:Carroll-1}
    n^{1/2}
    \left|
      F(r_{(k_n)}) - F(\zeta_{\alpha}) - k_n/n + E_n(\zeta_{\alpha})
    \right|\\
    \le
    n^{1/2}
    n^{-1}
    \sum_{i=1}^n
    \left|
      f(r_{(k_n)})
      -
      f(\zeta_{\alpha})
    \right|
    x_i (\hat w_n-w))
    =
    o(n^{1/2}a_n^2)
    \to
    0
    \quad
    \text{a.s.}
  \end{multline}
  The statement of Lemma~\ref{lem:Carroll} can now be obtained by plugging
  $$
    F(r_{(k_n)}) - F(\zeta_{\alpha})
    =
    (r_{(k_n)}-\zeta_{\alpha}) f(\zeta_{\alpha})
    +
    o(n^{-1/2})
    \quad
    \text{a.s.}
  $$
  (which follows from the second order Taylor expansion and Lemma~\ref{lem:2})
  and
  \begin{multline*}      
    E_n(\zeta_{\alpha})
    =
    F_n(\zeta_{\alpha})
    +
    n^{-1/2} W_n(0,a_n^{-1}(\hat w_n-w))\\    
    +
    n^{-1}
    \sum_{i=1}^n
    \Bigl(
      F(\zeta_{\alpha}+x_i(\hat w_n-w))
      -
      F(\zeta_{\alpha})
    \Bigr)
  \end{multline*}        
  (which follows from the definition of $W$)
  into (\ref{eq:Carroll-1}).
  Indeed, the addend involving $W_n$ is $o(n^{-1/2})$ a.s.\ by Lemma~\ref{lem:1}
  and, as we will see momentarily,
  \begin{equation}\label{eq:remains}
    n^{-1}
    \sum_{i=1}^n
    \bigl(
      F(\zeta_{\alpha}+x_i(\hat w_n-w))
      -
      F(\zeta_{\alpha})
    \bigr)
    -
    f(\zeta_{\alpha})
    \mu(\hat w_n - w)
    =
    o(n^{-1/2})
    \quad
    \text{a.s.}
  \end{equation}
  Therefore, it remains to prove~(\ref{eq:remains}).
  By the second order Taylor expansion,
  the minuend on the left-hand side of~(\ref{eq:remains})
  can be rewritten as
  \begin{multline}\label{eq:Carroll-2}
    n^{-1}
    \sum_{i=1}^n
    f
    \bigl(
      \zeta_{\alpha}
    \bigr)
    x_i(\hat w_n-w))
    +
    O
    \left(
      n^{-1}
      \sum_{i=1}^n
      x_i^2
    \right)
    o(a_n^2)\\
    =
    n^{-1}
    \sum_{i=1}^n
    f(\zeta_{\alpha})
    x_i(\hat w_n-w))
    +
    o(n^{-1/2})
    \quad
    \text{a.s.}
  \end{multline}
  where we have used
  $a_n^{-1}(\hat w_n-w)\to0$ a.s.\ (Lemma~\ref{lem:RR})
  and $\Expect x_1^2<\infty$.
  And the difference between the first addend of~(\ref{eq:Carroll-2})
  and the subtrahend on the left-hand side of~(\ref{eq:remains})
  is $O(n^{-1} a_n (n\ln\ln n)^{1/2})=o(n^{-1/2})$.
\end{document}